\documentclass[twoside,11pt]{article}

\usepackage{jmlr2e}

\usepackage{times}
\usepackage{DEF-bold}
\usepackage{DEF-script}
\usepackage[all]{xy}
\usepackage{color}
\usepackage{bbm}
\usepackage{algorithmic,xcolor}
\usepackage{algorithm}
\usepackage{booktabs}
\usepackage{multirow}
\usepackage{amsmath}
\usepackage{amssymb}
\usepackage{amsfonts}

\def\pa{{\rm pa}}
\def\pax{{\rm pa}}

\def\Ind{\mathbb{1}}

\def\V{\mathcal{V}}
\def\I{S}

\def\cone{\mathcal{C}}

\def\Pr{\mathbb{P}}
\def\Ex{\mathbb{E}}
\def\Ind{\mathbb{I}}
\def\bssw{\beta_{s,s\p}^{w}}
\def\tssw{\theta_{s,s\p}^{w}}
\def\lssw{\ell _{s,s\p}^{w}}
\def\e{{\rm e}}

\def\p{^\prime}

\def\d{{\rm d}}
\DeclareMathOperator*{\argmin}{argmin}
\ShortHeadings{Structure learning for CTBN's }{Shpak et. al.}
\firstpageno{1}

\begin{document}

\title{Structure learning for CTBN's via penalized maximum lieklihood methods}

\author{\name Maryia Shpak \email szpak.maria@poczta.umcs.lublin.pl\\
	\addr Faculty of Mathematics, Physics and Computer Science,\\ Maria Curie-Sklodowska University,\\
pl. Marii Curie-Sk?odowskiej 5,
20-031 Lublin, Poland \\
\AND
	\name B{\l}a{\.z}ej Miasojedow \email bmia@mimuw.edu.pl	 \\
        \addr Institute of Applied Mathematics, University of Warsaw\\ 
Banacha 2, 02-097 Warsaw,  Poland\\
\AND
\name Wojciech Rejchel \email  wrejchel@gmail.com\\
       \addr Faculty of Mathematics and Computer Science \\
       Nicolaus Copernicus University\\
       ul. Chopina 12/18,
87-100 Toru{\'n}, Poland }

\editor{}

\maketitle

\begin{abstract}%

The continuous time Bayesian networks (CTBNs) represent a class of stochastic processes, which can be used to model complex phenomena, for instance, 
 they can describe interactions occurring in living processes, in social science models or in medicine. The literature on this topic is usually focused on 
 the case, when the dependence structure of a system is known and we are to determine conditional transition intensities (parameters of the network). 
 In the paper, we study the structure learning problem, which is a more challenging task and the existing research on this topic is limited. 
 The approach, which we propose, is based on a penalized likelihood method. We prove that our algorithm, under mild regularity conditions, recognizes the 
 dependence structure of the graph with high probability.
 We also investigate the properties of the procedure in numerical studies to demonstrate its effectiveness .%

\end{abstract}


\section{Introduction}

%In almost every area of research we face various complex processes evolving over continuous time whose behaviour we need to explore. 
Learning the behaviour of complex processes, which evolve over continuous time, is a challenging task. 
One of the methods to describe such phenomena is the use of continuous time Bayesian networks (CTBNs) introduced by \citet{Nod1}.
Roughly speaking, a CTBN is a multivariate Markov jump process (MJP), whose dependence structure 
between coordinates can be described by a graph. Such a graphical representation 
allows for decomposing a large intensity matrix into smaller conditional intensity matrices. On the one hand, CTBNs are very flexible and can be used to model complex phenomena, for instance,
they can describe interactions between gene expressions in auto regulatory networks, enzymatic reaction graphs or correlations in social networks. 
On the other hand, the  modular structure allows for inference even in high-dimensional scenarios, 
for instance when the number of nodes in the graph is large with respect to the observation time. 
There is comprehensive literature concerning statistical inference for 
CTBNs. 
Most of them focus on the estimation of parameters for the known structure of a graph.
Such parameter learning for CTBNs in both the Bayesian and frequentist approach 
was studied in \citet{Nod2,Nod4}. Computational methods for CTBNs based on sampling were considered in \citet{EFK,FaSh,FarSher2006,hobolth2009,Nod2,RaoTeh2013a,Miasojedow2017}. 
Approaches relating to numerical approximations can be found in \citet{cohn2010mean,Nod1,Nod3,opper2008variational}.

% We need it in.... genetics gene regulatory networks demographic financial
In the current paper we consider the problem of structure learning, namely we want to find edges in a directed graph using the data. 
Learning such dynamic systems is a challenging task and the existing literature is modest.
The Bayesian approach with the score function  maximized by greedy algorithms is considered in \citet{Nod4,Acerbi_2014}, while 
the variational approach is studied in \citet{Linzner:2018}.
In the current paper, we propose using a penalized likelihood method to recognize the structure of the graph.
Similar approach was successfully applied to {\it static} graphical models with continuous and discrete variables 
\citep{Friedman_2007, banerjee08, geerbuhl11, Ravi10, HofTib09,Guoetal10, Xueetal12, Rejchel18}. 

In the proposed approach we consider the ,,full'' graph (i.e.~the 
graph with all possible edges) and we remove the spurious edges using the Lasso-penalized likelihood method. The Lasso penalty \citep{Tibsch96} is very useful and popular in the variable selection problem in the regression analysis
\citep{ els:01, geerbuhl11}. 
In the paper, we show that Lasso can be  applied successfully to 
{\it sparse} CTBNs, where {\it sparsity} means that the number of edges in the graph is relatively small compared to the number of nodes and the observation time.
To use the penalized likelihood method in CTBNs we introduce a new parameterization, 
namely for each node the conditional intensity matrix is modeled as the regression function 
in generalized linear models (GLM). The analogous approach can be found in \citet{andersen1982, Cox13}, 
where the Cox model is considered. We introduce  artificially explanatory variables (covariates) as {\it dummy variables} corresponding to  configurations of parents' states. 
We show that our procedure is able to recognize the structure of the graph under rather mild conditions. 
In \citet{1909.04570v3} one can find a similar approach, namely they also consider a full graph and then remove unnecessary edges.
However, our method used to remove edges is different. 
They use marginal posterior probabilities of the presence of edges, while we use the penalized likelihood. Moreover, the novelty of our approach is that we can give theoretical guarantees of consistency of the method, while other papers show efficiency only by simulations. 
To the best of our knowledge this is the first theoretical result on consistency of structure selection for CTBNs.
The main difficulty of the considered model is continuous time nature of the phenomena, which we investigate. 
Therefore, our argumentation is strongly based on martingale methods, for instance, martingale concentration inequalities. 
Finally, we also illustrate the quality of our method by numerical experiments on simulated data sets.

The rest of the paper is organized as follows. In Section \ref{sec:CTBN} we introduce the notion of CTBNs and its main characteristics. 
Section \ref{sec:structure} contains a detailed explanation of a proposed approach to learning the structure of the network. 
It also contains two main theoretical results (Theorem \ref{thm:consistency} and Corollary \ref{thm:consistency2}) which describe properties of the considered estimator.  
In Section \ref{sec:numerical} we investigate the behaviour of our procedure on simulated data sets. 
The paper is concluded in Section \ref{sec:discussion}. The proofs of the main results and auxiliary results are given in the appendix. 

\section{Continuous time Bayesian networks}\label{sec:CTBN}

Let $(\mathcal{V},\mathcal{E})$ denote a directed graph with possible cycles, where $\mathcal{V}$ is the set of nodes and $\mathcal{E}$ is the set of edges.  The notation  $w\to u$ 
means that there exists an edge from the node $w$ to the node $u.$
For every $w\in\mathcal{V}$ we consider a corresponding space  $\mathcal{X}_w$ of possible states at $w$ and we  
assume that each space $\mathcal{X}_w$ is finite.
We consider a continuous time stochastic process on the product
space $\mathcal{X}=\prod_{w\in\mathcal{V}} \mathcal{X}_w$, so a state $s\in\mathcal{X}$ is a
configuration $\mathbf{s}=(s_w)_{w\in\mathcal{V}}$, where $s_w\in\mathcal{X}_w$. 
If $\mathcal{W}\subseteq\mathcal{V},$ then we write $s_\mathcal{W}=(s_w)_{w\in\mathcal{W}}$ for the
configuration $s$ restricted to nodes in $\mathcal{W}$. We also use the notation
$\mathcal{X}_\mathcal{W}=\prod_{w\in\mathcal{W}} \mathcal{X}_w$, so  we can write $s_\mathcal{W}\in\mathcal{X}_\mathcal{W}$. In what follows
we use the bold symbol $\bf{s}$ to denote configurations belonging to $\mathcal{X}$ only. All restricted configurations will be denoted with standard font~$s$.  
%Set $\mathcal{W}\setminus\{w\}$ will be denoted by $\mathcal{W}-w$ and 

The set $\mathcal{V}\setminus\{w\}$ will be denoted 
by $-w$. Moreover, we define the set of parents of the node $w$ by 
\[\pa(w)=\{u\in\mathcal{V}\;:\;u\to w\}.\]
%and the set of children of the node $w$ by   
%\[\ch(w)=\{u\in\mathcal{V}\;:\;w\to u\}\;.\]
Suppose that for any fixed $w\in\mathcal{V}$ we have a function
$Q_w:\X_{\pa(w)}\times(\X_w\times \X_w)\to[0,\infty)$. More precisely,
for a fixed $c\in \X_{\pa(w)}$ we consider $Q_w(c;\cdot,\cdot\;)$ to be a conditional intensity matrix 
(CIM) at the node $w$ (only off-diagonal elements of this matrix
have to be specified, the diagonal ones are irrelevant).
The state of a CTBN at time $t$ is a random element $X(t)$ of the space 
$\X$ of all configurations. Let $X_w(t)$ denote its $w$-th coordinate. 
The process $\left\{(X_w(t))_{w\in\mathcal{V}}:t\geq 0\right\}$ 
is assumed to be Markov and its evolution can be described 
informally as follows: transitions at the node $w$ depend on the current 
configuration of its parents. If the states
of some parents change, then the transition probabilities (represented by CIM) at the node $w$ change.  Namely, if  $s_w\not=s_w\p ,$ then  
\begin{equation*}
%\label{informal}
         \Pr\left(X_w(t+\d t)=s_w\p|X_{-w}(t)=s_{-w},X_w(t)=s_w\right)=
              Q_w(s_{\pa(w)},s_w,s_w\p)\,\d t. 
\end{equation*}
Formally, a CTBN is a Markov jump process (MJP) with state space $\mathcal{X}$ and with transition intensities given by  
\begin{equation}\label{def: intensity}
    Q(\bf{s,s\p})=
          \begin{cases}
             Q_w(s_{\pa(w)},s_w,s_w\p), & \text{if $s_{w}\not=s_{w}\p$ and $s_{-w}=s_{-w}\p$ for some $w$;} \\        
              0,       &  \text{otherwise}\;,
          \end{cases}
\end{equation}
for $\mathbf{s,s\p}\in\mathcal{X}$, $\bf{s}\not=\bf{s\p}.$ Obviously, $Q(\bf{s,s})$ is defined ``by subtraction'' to ensure that $\sum\limits_{\bf{s\p}} Q({\bf{s,s\p}})=0$.  

For a CTBN the density  of a sample path  $X=X([0,T])$ on a bounded time interval $[0,T]$ decomposes as follows:
\begin{equation}
 \label{eq:densCTBN}
 p(X)=\nu(X(0))\prod_{w\in\mathcal{V}}p(X_w\Vert X_{\pa(w)})\;,
\end{equation}
where $\nu$ is the initial distribution on $\X$ and $p(X_w\Vert X_{\pa(w)})$ is the density of a piecewise 
homogeneous MJP with the intensity matrix equal to $Q_w(c;\cdot,\cdot\;)$ on every time sub-interval, 
where $X_{\pa(w)}=c$, so that  (see for example \citet{Nod4})
\begin{equation}\label{cbi}
       p(X_w\Vert X_{\pa(w)})=
             \prod_{c\in\X_{\pax(w)}}
                      \prod_{s\in\X_w} \prod_{s\p\in\X_w\atop s\p\not=s} 
                      Q_w(c;\; s,s\p)^{n_w^T(c;\; s,s\p)} \exp\left[-Q_w(c;\; s,s\p)
                              t_w^T(c;\; s)\right],
\end{equation}
where
\begin{itemize}
  \item[]  $n_w^T(c;\;s,s\p)$ denotes the number of jumps from 
$s\in\X_w$ to $s\p\in\X_w $ at the node $w$ on the time interval $[0,T]$ 
which occur when the parent configuration is $c\in\X_{\pax(w)}$,
\item[]  $t_w^T(c;\; s)$ is the length of time that the node $w$ is in the state $s \in \X_w$  on the time interval $[0,T]$   when the configuration of parents is $c\in\X_{\pax(w)}$.   
\end{itemize}
To simplify the notation, in the rest of the paper we omit the upper index $T$ in $n_w^T(c;\;s,s\p)$ and $t_w^T(c;\; s)$,  whenever it does not lead to confusion.

\section{Structure learning for CTBNs}
\label{sec:structure}

In this section, we describe the proposed method. As we have already mentioned our approach is 
to consider the full graph, namely we assume that $\pa(w) = -w$ for each $w \in \V$. Then we remove unnecessary edges using the penalized likelihood technique. 
We start by introducing the new parametrization of the model.
For simplicity, in the paper we consider the binary graph, i.e. $\X _w =\{0,1\}$ for each $w \in \V. $ The extension of our results to the general case is discussed in Section~\ref{sec:discussion}.

Let $d$ be the number of nodes in the graph.
Consider a fixed order $(w_1, w_2,\ldots, w_d)$ of nodes of the graph. Using this order we define a $ (2d) \times d$-dimensional matrix 
\begin{equation}
\label{beta}
\beta=\left(\beta_{0,1}^{w_1}, \beta_{1,0}^{w_1},
\beta_{0,1}^{w_2}, \beta_{1,0}^{w_2},
\ldots, \beta_{0,1}^{w_d}, \beta_{1,0}^{w_d}\right)^\top ,
\end{equation}
whose rows  are vectors $\bssw \in \mathbb{R}^d$
for all $w \in \V$ and $s,s' \in \{0,1\}$ such that $s\neq s'.$ Obviously, the matrix $\beta$ can be easily transformed to $2d^2$-dimensional vector in a standard way. 
In the paper we assume  that for all $w \in \V$, $\;c \in \X_{-w}$, $\;s,s' \in \{0,1\}$, $\;s\neq s'$ 
the conditional intensity matrices satisfy
\begin{equation}
\label{def: beta}
\log(Q_w(c,s,s\p))= {\beta_{s,s\p}^{w}}^{\top} Z_w(c)\;,
\end{equation}
where $Z_w\colon \X_{-w}\to \{0,1\}^{d}$ is a binary deterministic function. 
In \eqref{def: beta} the conditional intensity matrix $Q_w(\cdot,s,s')$ is modeled in the analogous way to the regression function in 
generalized linear models (GLM) and the functions $Z_w(\cdot)$ play roles of explanatory variables (covariates).
In our setting the link function is logarithmic.
The analogous approach can be found in \citet{andersen1982, Cox13}, where the Cox model is considered. 
The relation between the intensity and covariates in those papers is similar to \eqref{def: beta}.
Since the considered CTBNs do not contain explanatory variables, we introduce them artificially 
as any possible representations of parents' states. Thus, for every $w \in \mathcal{V}$ these
explanatory variables are {\it dummy variables} encoding all possible configurations in $\pa (w) = -w.$ 
To make it more transparent we consider the following example.

\begin{example}
\label{example_ctbn}
We consider a CTBN with three nodes $A,B$ and $C.$ For the node $A$ we define the function $Z_A$ as 
$$
Z_A(b,c)=[1,\Ind(b=1),\Ind(c=1)]^{\top}
$$
for each $b,c \in \{0,1 \}$, where $\Ind (\cdot)$ is the indicator function. Therefore, for 
each configuration of parents' states (i.e. values in nodes $B$ and $C$) 
the value of the function $Z_A(\cdot,\cdot)$ is a three-dimensional binary vector whose coordinates correspond to the intercept, 
the value in the parent $B$ and the value in the parent $C,$ respectively. Analogously, we define representations for remaining nodes
\begin{align*}
Z_B(a,c)&=[1,\Ind(a=1),\Ind(c=1)]^{\top},\\ 
Z_C(a,b)&= [1,\Ind(a=1),\Ind(b=1)]^{\top}
\end{align*}
for each $a,b,c \in \{0,1 \}.$  In this example the expression \eqref{beta} is defined as
$$\beta=\left(\beta_{0,1}^{A}, \beta_{1,0}^{A},
\beta_{0,1}^{B}, \beta_{1,0}^{B},
\beta_{0,1}^{C}, \beta_{1,0}^{C}\right)^\top \;.$$
With slight abuse of notation, the vector $\beta ^A_{0,1}$ is given as 
$$\beta^A_{0,1}=\left[\beta^A_{0,1} (1), \beta^A_{0,1} (B), \beta^A_{0,1} (C) \right]^ \top .$$
and we interpret \eqref{def: beta} in the natural way:  $\beta^A_{0,1} (B) = 0$ means that the intensity of the change from the state $0$ to $1$ at 
the node $A$ does not depend on the state at the node $B.$ Similarly, $\beta^A_{0,1} (C) $ describes the dependence between the above intensity and the state at 
the node $C,$ and  $\beta^A_{0,1} (1)$ corresponds to the intercept. For the node $B$ the coordinates of the vector 
$$\beta^B_{0,1}=\left[\beta^B_{0,1} (1), \beta^B_{0,1} (A), \beta^B_{0,1} (C) \right]$$ describe the relation between the intensity of the jump 
from the state $0$ to $1$ at the node $B$ to the intercept, states at nodes $A$ and $C,$ respectively.
\end{example}

Analogously as in Example~\ref{example_ctbn}, for $w\in\V$, $u\not=w$, and $s,s'\in\{0,1\}$, $s\not=s'$ by $\beta_{s,s'}^{w}(u)$ we denote a coordinate of the vector $\beta_{s,s'}^{w}$
corresponding to the node $u$. We interpret $\beta_{s,s'}^{w}(u)$ as the parameter describing the dependence of the intensity of the jump from the state $s$ to $s'$ at the node $w$ 
on the state at $u$.

Our goal is to find edges in a directed graph $(\V,\E). $ 
We define the relation between edges in $(\V,\E) $ in the following way 
$$
  \beta^w_{0,1} (u) \neq 0 \;  {\rm or}\; \beta^w_{1,0} (u) \neq 0 \;
\Leftrightarrow \; {\rm the \; edge} \; u\to w \; {\rm exists},
$$ 
which makes them compatible with the considered CTBNs.
 Roughly speaking, the fact that the node $u$ is a parent of $w$ means that
 the intensity of switching a state at $w$ depends on the value at the state at $u$. 
 Therefore, the problem of finding edges in the graph is reformulated as the problem of estimation of the parameter~$\beta.$

\begin{remark}
\label{remark_intercept}
For simplicity, in the rest of the paper, we omit the first coordinate 
$\bssw (1)$ in the vector $\bssw$ for all $w,$ $s\neq s',$ because it corresponds to the intercept and is not involved in recognition of the edges in the graph. 
The first coordinates of representations $Z_w(c)$ are discarded as well. 
\end{remark}

Our method is based on estimating the parameter $\beta$ using the penalized likelihood method. In the rest of the paper the term $\beta$ is reserved for the true value of the parameter. Another quantities are denoted by $\theta.$ First, we consider a function 
\begin{equation}\label{def: loglik_beta}
\ell(\theta)=\frac{1}{T} \sum_{w\in\V}\sum_{c\in\X_{-w}} \sum_{ s\not=s'}\left[-n_w(c;\; s,s\p){\theta_{s,s\p}^{w}}^{\top} Z_w(c)+t_w(c;\; s)\exp\left({\theta_{s,s\p}^{w}}^{\top} 
Z_w(c)\right)\right],
\end{equation}
where the third sum in \eqref{def: loglik_beta} is over all $s,s' \in \X _w$ such that $s \neq s'.$ Notice that the function \eqref{def: loglik_beta} is the negative log-likelihood. 
Indeed, we just apply the minus logarithm to the density 
\eqref{eq:densCTBN} combined with \eqref{cbi} and \eqref{def: beta}, where $\pa (w) = -w$ for each $w \in \V.$ Then we divide it by $T$ and omit the term corresponding 
to the initial distribution $\nu,$ because $\nu$ does not depend on $\beta .$ 
We define an estimator of $\beta$ as
\begin{equation}\label{minimizer}
	 \hat \beta =\argmin_{\theta\in \mathbb{R} ^{2d(d-1)}} \left\{ \ell(\theta)+\lambda |\theta|_1\right\}\;,
\end{equation}
where  $|\theta|_1 = \sum\limits_{w\in\V} \sum\limits_{ s\not=s'} \sum\limits_{u \in -w} 
|\theta_{s,s\p}^w (u)|$ is the $l_1$-norm of $\theta.$ The tuning parameter $\lambda >0$ is
a balance between minimizing the negative log-likelihood and the penalty. 
The form of the penalty is crucial, because
its singularity at the origin implies that some coordinates of the minimizer $\hat \beta$ are exactly equal to zero, if $\lambda$ is sufficiently large. 
Thus, starting from the full graph we remove irrelevant edges and estimate parameters for existing ones simultaneously.
The function $\ell (\theta)$ and the penalty are convex, so \eqref{minimizer} is a convex minimization problem, 
that is an important fact from both practical and theoretical point of view. 

At first glance, computing \eqref{minimizer} seems to be computationally complex, because 
the number of summands in \eqref{def: loglik_beta} is $d2^d.$ However, the number of nonzero 
 $n_w(c;\; s,s\p)$ and $t_w(c;\; s)$ is bounded by total number of jumps, which grows linearly with time $T$.
Hence, most of summands in \eqref{def: loglik_beta} are also zeroes and the minimizer \eqref{minimizer} can be calculated efficiently. 

\subsection{Notations}
\label{subsec:notations} 
In the rest of the paper we need additional notation. Most of them are collected in this subsection.
First, for each $w \in \V$ we denote its parents indicated by the true parameter $\beta$ as
\begin{equation}
\label{Sw}
S_w=\left\{u \in -w: \beta^w_{0,1} (u) \neq 0 \quad or \quad \beta^w_{1,0} (u) \neq 0\right\}.
\end{equation}
By $S$ we denote the support of $\beta,$ i.e.~the set of nonzero coordinates of $\beta.$
Moreover, $\beta _{\min}$ is the smallest (in absolute values) element of $\beta$ restricted to $S$. 
The set $S^c$ denotes the complement of $S$, that is the set of zero coordinates of $\beta.$ Besides, for each $w \in \V$ we define $-S _ w = \V \setminus \{S _w\cup w\}$ and denote $\Delta = \max\limits_{\bf{s \neq s\p}}Q(\bf{s,s\p}).$ 

For a vector $a$ we denote its $l_\infty$-norm by $|a|_\infty = \max\limits_k |a_k|.$ For a subset $\mathcal{A}$ the vector $a_{\mathcal{A}}$ denotes a vector such that
$(a_{\mathcal{A}})_k=a_k$ for $k\in\mathcal{A}$ and $(a_{\mathcal{A}})_k=0$ otherwise. Moreover, $|\mathcal{A}|$ denotes the number of elements of $\mathcal{A}.$

%Next, for each $w \in \V$ we consider a new MJP, which lives on 
%$ S _w \cup w$ and its state space is $\X _{\I _w \cup w} \cup \{\ast \},$
%where $\ast$ is an additional state. The intensity matrix of the new process is denoted by $Q_{\I _w }$ and defined as: for every $s, s' \in \X _{\I _w \cup w}\cup \{\ast\}$ such that $s \neq s'$ we have
%\begin{equation}
%\label{Qiw}
%Q_{\I _w }(s,s\p)= \begin{cases}
%Q((s,0),(s\p,0))& \text{if $s,s\p\neq \ast$, $s\neq s\p$,}\\
%\sum\limits_{c_{-\I _w}\in\X_{-\I _w}\setminus\{0\}} Q((s,0),(s, c_{-\I _w}))& \text{if $s\neq\ast$ and $s\p =\ast$,}\\ \sum\limits_{c_{-\I _w}\in\X_{-\I _w}\setminus\{0\}} Q((s\p,c_{-\I _w}),(s\p,0 ))& \text{if $s=\ast$ and $s\p \neq\ast$.}
%\end{cases}
%\end{equation} 
%Thus, the new process is similar to the marginal MJP, which is restricted to  $ S _w \cup w$. However, the difference is that we extend its state space by adding the new state $\ast$. This state 
%,gathers'' all cases of the original process, which correspond to nonzero configurations on $-S_w.$
%The concatenation $(s,c_{-\I _w})$ in \eqref{Qiw}  is a $d$-dimensional vector defined in such a way that its coordinates appear in the ,,appropriate order''. We assume that the order
%of coordinates in the remaining  concatenations is also correct. The original process is irreducible and aperiodic, so new processes are also irreducible and aperiodic.
%For each $w \in \V$ their 
Let $\pi$ be the stationary distribution of the  MJP, which is defined by $Q.$ 
The initial distribution of this process is denoted by $\nu$ and we define
\[
\Vert\nu \Vert_2^2 = \sum\limits_{s\in \X}\nu^2 (s)/\pi^2(s)%+\nu^2 _w(\ast)/\pi^2_{\I _w}(\ast).
   .  \]
Moreover, $\rho_1$ denotes the smallest positive eigenvalue 
of $-1/2(Q+Q^*)$, where $Q^*$ is an adjoint matrix of $Q_.$

\subsection{Main results}
\label{subsec:main}

In this subsection, we state key results of the paper. In the first one (Theorem \ref{thm:consistency}) we show that the estimation error of the minimizer \eqref{minimizer} 
can be controlled. In the second result (Corollary \ref{thm:consistency2}) we state that
the thresholded version of \eqref{minimizer} is able to recognize the structure of the graph.

First, we introduce the cone invertibility factor (CIF), which plays an important role in the theoretical analysis of properties of Lasso estimators. 
 Our goal is to show that the estimator $\hat{\beta}$ is close to the true $\beta$. To do it, we show in Lemma \ref{lem:lambda} in the appendix that the gradient of the likelihood (\ref{def: loglik_beta}) at $\beta$ is close to zero. However, it is not enough. Namely, the likelihood function cannot be too ,,flat''. In the high-dimensional scenario it is often provided by assuming the restricted strong convexity condition (RSC) on (\ref{def: loglik_beta}), as in \citet{negahban12}. The cone invertibility factor defined in \eqref{Fbar} plays a similar role to RSC, but gives sharper consistency results \citep{YeZhang10}. Therefore, it is
used  in the paper.
 CIF is defined analogously to \citet{YeZhang10, HuangGLM12, Cox13} 
that concerns linear regression, generalized linear models and the Cox model, respectively. 
It is also closely related to the compatibility factor \citep{Geer2008} or the restricted eigenvalue condition \citep{Bickel09}.
Thus, for $\xi>1$ and the set $S,$ which denotes the support of $\beta,$ we define a cone as
\[\cone (\xi,S) = \left\{\theta: |\theta_{S^c}|_1 \leq \xi |\theta_{S}|_1\right\}\,. \]
The cone invertibility factor is defined as 
\begin{equation}
\label{Fbar}
\bar{F} (\xi)= \inf _{0 \neq \theta \in \cone (\xi,S)} \frac{ \theta ' \nabla ^2 \ell(\beta) \theta}{|\theta _S|_1 |\theta|_\infty} .
\end{equation}
Notice that only the value of the Hessian $\nabla ^2 \ell(\theta)$ at the true parameter $\beta$ is taken into consideration in \eqref{Fbar}. 
The main difficulty with CIF in our case is that it is a sum over exponentially many in $d$ random terms. To be able to control it, we lower bound it by the deterministic value with much fewer summands. In Lemma~\ref{lem:cif} in the appendix we prove that \eqref{Fbar} is lower bounded by the multiplication of $\zeta$ given in Theorem \ref{thm:consistency} and
\begin{equation}\label{ass:cif}
F(\xi) =\inf_{0\not = \theta \in C(\xi,S)}\sum_{w\in\V}\sum_{s\p\not=s}\sum_{c_{\I _w}\in\X_{\I _w}}
\frac{\exp\left(\beta_{s,s\p}^{w\top} Z_w(c_{S_w},0)\right)\left[\theta_{s,s\p}^{w\top}  Z_w(c_{S_w},0)\right]^2}{|\theta_S|_1 |\theta|_\infty}
\end{equation}
with probability close to one.  Note, that in \eqref{ass:cif} we restrict summation only to $c_{S_w}\in\X_{S_w}$ by taking $c_{-S_w}=0$. This allows us to lower bound $\bar{F}(\xi)$ without considering exponentially many, in $d$,  random summands. Our argumentation will also follow  in the case, when we choose some nonzero values as $c_{-S_w}$, unless this value does not depend on $w$ and $c_{S_w}.$

Now we can state two main results of the paper.
\begin{theorem}
\label{thm:consistency}
Let $\varepsilon \in (0,1), \xi >1$ be arbitrary. Suppose that $F(\xi)$ defined in \eqref{ass:cif} is positive and
\begin{equation}
\label{Tform2} 
T>
\frac{36 \left[ (\max\limits_{w \in \V} |\I _w| +1) \log 2 + \log\left(
d  ||\nu||_2 /\varepsilon
\right) 
\right]}{\min\limits_{w \in \V,s \in\X_w ,c_{\I _w}\in\X_{\I _w}} \pi^2(s,c_{\I _w},0)\rho_1} .
\end{equation}
We also assume that $T \Delta \geq 2$ and 
\begin{equation}
\label{lambda_form}
2\frac{\xi+1}{\xi-1}\log(K/\varepsilon)\sqrt{\frac{\Delta}{{T}}}
\leq \lambda \leq \frac{2 \zeta F(\xi)}{e(\xi+1)|S|}\:,
\end{equation}
where
 $
K=2(2+e^2)d(d-1)$ and  $ \zeta=\min\limits_{w \in \V, s \in \X _w,c_{\I _w}\in\X_{\I _w}} \pi(s,c_{\I _w},0)/2.
 $ 
Then with probability at least $1-2\varepsilon$ we have
\begin{equation}
\label{estim_formula}
 |\hat\beta -\beta|_\infty\leq \frac{2e \xi \lambda}{(\xi+1)\zeta F(\xi)} \:.
\end{equation}
\end{theorem}

\begin{corollary}
\label{thm:consistency2}
Suppose that assumptions of Theorem \ref{thm:consistency} are satisfied.
Let $R$ denote the right-hand side of the inequality \eqref{estim_formula}.
Consider the thresholded Lasso estimator with the set of nonzero coordinates 
$\hat{S}.$ The set $\hat{S}$ 
contains only those coefficients of the Lasso estimator \eqref{minimizer}, which are  larger in the absolute value than a pre-specified threshold $\delta.$ 
If $ \beta _{min}/2> \delta \geq  R,$ then 
\[
P\left( \hat{S} = S \right) \geq 1- 2 \varepsilon\,.
\]
\end{corollary}

The above two results describe the properties of the proposed estimator \eqref{minimizer} in recognizing the structure of the graph. Theorem \ref{thm:consistency} gives conditions under which the estimation error of \eqref{minimizer} can be controlled. Namely, let us forget  about constants, $\Delta$ and parameters of  MJP, i.e. $\nu,\pi,\rho_1, \zeta$ etc. in assumptions.  
Then the estimation error is small, if we have that 
\begin{equation}
\label{assu_d}
T \geq \frac{\log ^2(d/\varepsilon) |S|^2 }{F^2(\xi)}\:
\end{equation}
by condition \eqref{lambda_form}. 
It states restrictions on the number of vertices in the graph, sparsity of the graph (i.e. the number of edges has to be small enough) and the expression \eqref{ass:cif}. The last term is discussed in Lemma \ref {cif_bound} (below). The condition \eqref{assu_d} is similar to standard results for Lasso estimators in \citet{YeZhang10, geerbuhl11, HuangGLM12, Cox13}. 
The only difference is that  the right-hand side of \eqref{assu_d} usually depends linearly on $\log (d/\varepsilon)$, but here we have $\log ^2 (d/\varepsilon).$ 
The square in the logarithm could be omitted, if we impose additional assumptions on  obervation time $T$ in the crucial auxiliary result (Lemma \ref{lem:lambda} 
in the appendix),  where we use the Bernstein-type inequality for the Poisson random variable. 
Obviously, it would reduce the applicability of the main result. In our opinion, the gain (having $\log (d/\varepsilon)$ instead of $\log ^2 (d/\varepsilon)$) is ,,smaller'' than the price (additional assumptions), so we do not focus on it.

The next assumption in Theorem \ref{thm:consistency} that $T \Delta \geq 2 $ is natural, because  observation time has to increase, when the maximal 
intensity of transitions decreases. Moreover, conditions \eqref{Tform2} and \eqref{lambda_form} depend also on parameters of MJP.   Precisely, they depends on
  the stationary distribution $\pi$ and the spectral gap $\rho_1,$ which in general decrease exponentially with $d.$ 
However, in some specific cases, it can be proved that they decrease polynomially.

Corollary \ref{thm:consistency2}  states that the  Lasso estimator after thresholding is 
able to recognize the structure of a graph with probability close to one,
if the nonzero coefficients of $\beta$ are not too close to zero and the threshold $\delta$ is appropriately chosen. 
However, Corollary \ref{thm:consistency2} does not give a way of choosing the threshold $\delta$, because both endpoints of the interval $[R, \beta _{\min}/2] $ are unknown. 
It is not a surprising fact and has been already observed, for instance, in linear models \citep[Theorem 8]{YeZhang10}. In the 
experimental part of the paper, we propose a method of choosing a threshold, that relates to information criteria. A similar procedure can be found in \citet{pokmiel:15, Rejchel18}.

Now we state a lower bound for \eqref{ass:cif}, which can be nicely interpreted.
\begin{lemma}
\label{cif_bound}
For every $\xi >1$ we have
\begin{equation}
\label{cif1}
F(\xi) \geq \frac{1}{\xi A_\beta }\:,
\end{equation}
where 
\begin{equation}
\label{BB}
A_\beta= \sum\limits_{w\in\V} \sum\limits_{s\p\not=s}\sum\limits_{j:\beta_{s,s\p}^{w} (j) \neq 0 }
\exp\left(-\beta_{s,s\p}^{w} (j) \right).
\end{equation}
\end{lemma}
Notice that the term $A_\beta$ decreases, if negative coefficients of $\beta$ ,,dominate'' positive ones. This situation means that our process ,,stucks'', because intensities in \eqref{def: beta} tend to be close to zero. Such behaviour in the context of MJPs is natural, because 
multiplying the intensity matrix $Q$ by constant $\kappa$ is equivalent to considering $T/\kappa$ instead of $T$. While we use $F(\xi)$ to lower bound 
$T$ such dependence on $\beta$ is expected.

Proofs of two main results can be found in the appendix. They are based on well-known facts for Lasso estimators (Lemmas \ref{basiclem} and \ref{estim}) as well as on new ones (Lemmas \ref{lem:lambda} and \ref{lem:cif}). The main novelty and difficulty of the considered model is continuous time nature of the observed phenomena, which we investigate. 
 In Lemma \ref{lem:lambda}  we derive the new concentration inequality for MJPs, which is based on the martingale theory. In Lemma \ref{lem:cif} we give new upper bounds on occupation time for MJPs.

\section{Numerical examples}
\label{sec:numerical}

In this part of the paper, we describe the details of algorithm implementation as well as results of experimental studies.

\subsection{Details of implementation}\label{implementation}
In this section, we provide in details practical implementation of the algorithm, which is proposed in the paper.
The solution of the problem \eqref{minimizer} depends on the choice of $\lambda$ in the penalty. 
Finding the ``optimal'' penalty parameter $\lambda$ and the threshold $\delta$  is difficult in practice. In the paper, we solve it using the information criteria \citep{Xueetal12, pokmiel:15, Rejchel18}. 

First, we observe that the function, which is minimized in \eqref{minimizer}, is a sum over  $w\in\V$, $s,s\p\in\{0,1\}$, $s\not=s\p$  of functions,
which depend only on $\theta_{s,s\p}^{w},$ i.e. the vector  $\theta$ {\it restricted to} coordinates corresponding to  $w,s\neq s\p$. So, for each triple $w,s\neq s\p$ we can solve the problem separately. In our implementation we use the following scheme. 
We start with computing a sequence of minimizers on the grid, i.e.~ for any triple $w\in\V$, $s \neq s\p$ we create a sequence $\{\lambda_i\}_{i=1}^{100}$ uniformly spaced on the log scale, starting from the largest $\lambda_i$, which 
corresponds to the empty model. Next, for all $\lambda_i$ we compute the estimator 
 \begin{equation}
\label{lassosw}
\hat{\beta}_{s,s\p}^{w}(i) = \argmin_{\theta_{s,s\p}^{w}} \left\{\ell_{s,s\p}^w(\theta_{s,s\p}^{w})+\lambda_i|\theta_{s,s\p}^{w}|_1\right\}\,,\end{equation}
 where
 \[
 \ell_{s,s\p}^w(\theta_{s,s\p}^{w})=\frac{1}{T} \sum_{c\in\X_{-w}}\left[-n_w(c;\; s,s\p){\theta_{s,s\p}^{w}}^{\top} Z_w(c)+t_w(c;\; s)\exp\left({\theta_{s,s\p}^{w}}^{\top} 
 Z_w(c)\right)\right]\,.\]
 To numerically solve \eqref{lassosw} for a given  $\lambda_i$ we use the FISTA algorithm with backtracking from \citet{FISTA}.
 The final Lasso estimator $\hat{\beta}_{s,s\p}^{w}:=\hat{\beta}_{s,s\p}^{w} ({i^*})$ is chosen using  the Bayesian Information Criterion (BIC), which is a popular method of choosing $\lambda$ in the literature  \citep{Xueetal12, Rejchel18}, i.e.
\[
 i^*=\argmin_{1 \leq i \leq 100} \left \{n \ell_{s,s\p}^w(\hat{\beta}_{s,s\p}^{w}(i))+\log(n)\Vert \hat{\beta}_{s,s\p}^{w}(i) \Vert_0\right\}\,,
\]
where $\Vert \hat{\beta}_{s,s\p}^{w}(i)\Vert_0$ denotes the  number of non-zero elements of  $\hat{\beta}_{s,s\p}^{w}(i)$ and $n$ is the number of observed jumps of the process.

 Finally, the threshold is obtained using the Generalized Information Criterion (GIC). A similar way of choosing a threshold was used previously in \citet{pokmiel:15, Rejchel18}.
For  a prespecified sequence of thresholds $\Omega$ we calculate  
\[
\delta^* =\argmin_{\delta \in \Omega} \left \{n\ell_{s,s\p}^w( \hat{\beta}_{s,s\p}^{w,\delta})+\log(2d(d-1))\Vert \hat{\beta}_{s,s\p}^{w,\delta} \Vert_0\right\}\;,
\]
where $\hat{\beta}_{s,s\p}^{w,\delta}$ is the Lasso estimator $\hat{\beta}_{s,s\p}^{w}$ after thresholding with the level $\delta.$

\subsection{Simulated data}
We consider two models defined as follows:
\begin{itemize}
\item[M1]  All vertices have the ``chain structure'', i.e.~for any node, except for the first one, its set of parents contains only a previous node. Therefore, we have  $\V = \{1,\dots,d\}$ and
    $\pa(k)= k-1,$ if $k>1$ and $\pa(1)=\emptyset$. We construct  CIM in the following way. For the first node the intensities of leaving both states are equal to $5$. 
    For other nodes $k$, $k>1$  we choose randomly $a\in\{0,1\}$ and we define

\begin{equation}\label{int_chain}
  Q_k(c,s,s\p) =\begin{cases}
                 9 &\text{if $s\not=|c-a|,$}\\
                 1 &\text{if $s=|c-a|.$}
                \end{cases}
\end{equation}
In words, we choose randomly,  if the  node  prefers to be at the same state as its parent or not. Say that the node $k$ prefers to be at the same state as the node  $k-1$, then
if these two states coincide the intensity of leaving the current state is $1$, otherwise it is $9$.  The intensity is  defined  analogously, when the node $k$ does not prefer to be at the same state as the node $k-1$.

\item[M2]  The first $5$ vertices are correlated, while the remaining vertices are independent. 
 We sample $10$ arrows between first $5$ nodes by choosing randomly $2$ parents for each node. We define intensities as follows
 \begin{equation}\label{int_example}
  Q_w(c,s,s\p) =\begin{cases}
                 5 &\text{if $\pa(w)=\emptyset$,}\\
                 9 &\text{if $\pa(w)\not=\emptyset$, $s$ is preferred state and $\prod_{c_i\in c}c_i=1,$ }\\
                 1 &\text{if $\pa(w)\not=\emptyset$, $s$ is preferred state and $\prod_{c_i\in c}c_i=0,$}\\
                 9 &\text{if $\pa(w)\not=\emptyset$, $s$ is not preferred state and $\prod_{c_i\in c}c_i=0,$ }\\
                 1 &\text{if $\pa(w)\not=\emptyset$, $s$ is not preferred state and $\prod_{c_i\in c}c_i=1,$}\\
                \end{cases}
 \end{equation}
  where the preferred state is chosen randomly from $\{0,1\}$. In words, for every node $w$ with $\pa(w)\not = \emptyset$ we choose randomly one state, say $0$. 
  In this case, if all parents are $1$ the process prefers
  to be in $1$ and if some of the parents are $0$ the process prefers to be in $0$.
 \end{itemize}
The model $M1$ has a simple structure which involves all vertices and satisfy our assumption~\eqref{def: beta}. 
The model $M2$ corresponds to a dense structure on a small subset of vertices. In addition,
the model $M2$ does not satisfy assumption~\eqref{def: beta}. Another potential difficulty is related to possible feedback loops, which are usually hard to recognize.
\begin{table}[ht]
\centering
\begin{tabular}{lll rrr }
  \toprule
 Model  & d & Time & Power & FDR & MD\\
  \midrule
M1 & 20   & 10 &0.93& 0.21 & 22.4\\         
   &      & 50 &0.95& 0.07 &19.3 \\ 
   & 50   & 10 &0.86 &0.32 &61.7\\
   &      & 50 &0.88& 0.13 & 49.4\\
\midrule
M2 & 20   & 10 &0.22& 0.65&  7.05\\ 
   &      & 50 &0.27&0.42  &5.06\\
 \bottomrule
\end{tabular}
\caption{Results for simulated data.  In the model $M1$ the true
dimension is $19$ for $d=20$ and $49$ for $d=50$. In the model $M2$ the true model dimension is $10$. }
\label{tab:results}
\end{table}

We consider the following cases: $d=20,50$ for $M1$ and $d=20$ for $M2$. So, the considered  number of possible parameters of the model (the size of $\beta$) is $ 2d^2 = 800, 5000$, respectively. 
We use $T=10,50$ for both models and we replicate simulations $100$ times for each scenario. In the Table~\ref{tab:results} we present averaged results of simulations in terms of 
\begin{eqnarray*}
{\rm Power}&=&\frac{{\rm the \; number \; of\; correctly\; selected\; edges}}{{\rm the \;  number \;  of \; edges\; in \; the \; graph}}\:,\\
{\rm False \;discovery \;rate \;(FDR)} &=& \frac{{\rm the \;number \;of\; incorrectly \;selected \; edges}}{\max({\rm the\; number \; of \;  selected \;  edges},1)}\:,\\
{\rm Model \; dimension \; (MD)}&=& {\rm the \; number \; of \;selected \; edges}.
\end{eqnarray*}

We observe that in the  model $M1$ the results of experiments confirm that the proposed method works in a satisfactory way. For observation time
$T=10$ the algorithm has high power and its FDR is not large. The final model, that is selected by our procedure, is slightly too big (it contains a few non-existing edges). When we increase observation time ($T=50$), then our estimator behaves almost perfectly. 

The  model $M2$ is much more difficult and this fact has impact on simulation results. Namely, for $T=20$  the power of the algorithm is relatively low and FDR is large. The procedure performs slightly better, when we take $T=50.$ However, for both observation times the estimator cannot find the true edges in the graph. 
 One of the reason of such behaviour of the estimator  is that in $M2$ the dependence structure in CIM is not additive in parents. This fact
combined with possible feedback loops  leads to recovering existing edges, but having the opposite to the true ones directions. Looking deeper into the results for a few examples chosen from our experiments we confirm this claim, i.e.~the 
edges between nodes are correctly selected, but their directions are wrong. Therefore, we can conclude that in the complex model $M2$ our estimator seems at least to be able to recognize interactions between nodes, which is important in many practical problems on its own.  

\section{Discussion}
\label{sec:discussion}
In the current paper, we propose the method for structure learning of CTBNs. We confirm the good quality of  our method  both theoretically and experimentally.
To simplify notation and help the reader to follow our reasoning we restrict ourselves to binary graphs. However, our results could be straightforwardly generalized to finite graphs by 
extending $\beta$ to other possible jumps and possible values of parents. In terms of the explanatory variable, it is equivalent to the standard encoding of qualitative variables in 
linear or generalized linear models. Our results can be also easily generalized for the case, where we consider not only additive effect in \eqref{def: beta}, but also interactions between parents.

One of the most interesting question for the future research is, whether  our method can be adapted to partially observed and noisy data. In the case of partial observations we need to introduce 
the observation $Y$ and the likelihood function $g(y|x)$, which is the likelihood of the observed data $y$ given a hidden trajectory of a process $x$. We can again parametrize CIM by \eqref{def: beta}.
However, in this case the problem  \eqref{minimizer} becomes more challenging, since the negative log-likelihood is given by
\[
 \ell(\beta)=-\log\left(\int g(y|x)p_\beta(x)\right)dx\;,
\]
where $p_\beta(x)$ is given by \eqref{eq:densCTBN}. This definition leads to the following two problems. First, the theoretical analysis becomes challenging, 
because the loss function is not convex. Secondly, the function $\ell$ is also difficult  from the computational perspective.
We have a partial solution to the computational part of the problem. Namely, we can formulate the EM algorithm for this case, where the expectation step is the standard 
E-step and in the M-step we can proceed in  exactly  the same way as in the current paper. Since the density belongs to the exponential family, the E-step requires to 
compute the expected values of sufficient
statistics. It could be done using the numerical integration proposed by \cite{Nod4} or the MCMC algorithm developed in \cite{RaoTeh2013a}. In addition, the results from \citet{Majewski2018} or \citet{Davis2020} combined with \citet{Miasojedow2017} could be helpful in the analysis of the Monte Carlo scheme. 
This problem should be investigated thoroughly.

% Acknowledgments---Will not appear in anonymized version
\acks{The authors are supported by Polish National Science Center grant: NCN UMO-2018/31/B/ST1/00253. }

%\bibliography{yourbibfile}
\bibliography{refs}

\appendix

\section{Auxiliary results}

This section contains lemmas that are needed to prove the main results of the paper.

\begin{proof}[of Lemma~\ref{cif_bound}]
Fix $\xi>1.$  For each $w ,  c_{\I _w} $
we have $Z_w(c_{S_w},0)= (c_{S_w},0),$ so 
$$
F(\xi) =\inf_{0\not = \theta \in C(\xi,S)}\sum_{w\in\V}\sum_{s\p\not=s}\sum_{c_{\I _w}\in\X_{\I _w}}
\frac{\exp\left((\beta_{s,s\p}^{w})_{S_w}^\top c_{S_w}\right)\left[(\theta_{s,s\p}^{w})^\top_{S_w}  c_{S_w}\right]^2}{|\theta_S|_1 |\theta|_\infty}\:,
$$
 where $\left(\beta_{s,s\p}^{w} \right)_{S_w}$ and $\left(\theta_{s,s\p}^{w} \right)_{S_w}$ are  restrictions of $\beta_{s,s\p}^{w} $ and $\theta_{s,s\p}^{w} $  to coordinates from 
$S_w,$ respectively. To establish \eqref{cif1} we show that for each $\theta \in C(\xi,S) $ and 
$\theta \neq 0$ the expression
\begin{equation}
\label{cif2}
\frac{\sum\limits_{w\in\V} \sum\limits_{s\p\not=s}\sum\limits_{c_{\I _w}\in\X_{\I _w}}
\exp\left((\beta_{s,s\p}^{w})_{S_w}^\top c_{S_w}\right)\left[(\theta_{s,s\p}^{w})^\top_{S_w}  c_{S_w}\right]^2}{|\theta_S|_1 |\theta|_\infty}
\end{equation}
is lower bounded by the right-hand side of \eqref{cif1}. First, we restrict the third sum in the numerator of \eqref{cif2} to the  summands corresponding  only to vectors $e_i \in 
\X_{\I _w}$ having one on the $i$-th cooridinate and zeroes elsewhere. Doing that we decrease the numerator of \eqref{cif2} to
\begin{equation}
\label{cif3}
\sum\limits_{w\in\V} \sum\limits_{s\p\not=s}\sum\limits_{j \in S_w}
\exp\left(\beta_{s,s\p}^{w} (j) \right)\left[\theta_{s,s\p}^{w} (j)  \right]^2\, .
\end{equation}
Recall that  $S_w=\left\{u \in -w: \beta^w_{0,1} (u) \neq 0 \quad or \quad \beta^w_{1,0} (u) \neq 0\right\}.$ Therefore, if $\beta_{s,s\p}^{w} (j) \neq 0,$ then 
$j \in S_w, $ so \eqref{cif3} can be lower bounded by
\begin{equation}
\label{cif4}
\sum\limits_{w\in\V} \sum\limits_{s\p\not=s}\sum\limits_{j:\beta_{s,s\p}^{w} (j) \neq 0 }
\exp\left(\beta_{s,s\p}^{w} (j) \right)\left[\theta_{s,s\p}^{w} (j)  \right]^2\, ,
\end{equation}
because \eqref{cif3} has more summands and the summands are nonnegative. Using reverse H\"{o}lder's inequality we replace \eqref{cif4} by 
\begin{equation}
\label{cif5}
A_\beta^{-1}
\left[
\sum\limits_{w\in\V} \sum\limits_{s\p\not=s}\sum\limits_{j:\beta_{s,s\p}^{w} (j) \neq 0 } 
|\theta_{s,s\p}^{w} (j)|  \right]^2,
\end{equation}
 where $A_\beta$ is defined in \eqref{BB}. Next, recall that $S$ is the set of nonzero coordinates of $\beta,$ so \eqref{cif5} is just $ |\theta_S|_1^2 / A_\beta.$ Summarizing, we lower bound \eqref{cif2} by 
\begin{equation}
\label{cif6} \frac{|\theta_S|_1}{  A_\beta |\theta|_\infty}
\end{equation}
for each $\theta \in C(\xi,S) $ and 
$\theta \neq 0.$
The vector $\theta$ belongs to the cone  $C(\xi,S) ,$ which implies that 
$$|\theta_{S^c}|_\infty \leq |\theta_{S^c}|_1  \leq \xi |\theta_{S}|_1 $$ and
$$|\theta|_\infty = \max(|\theta_{S}|_\infty, |\theta_{S^c}|_\infty) \leq 
\max( |\theta_S|_1, \xi |\theta_{S}|_1),$$ 
which gives us $|\theta|_\infty \leq \xi |\theta_{S}|_1 .$ Applying it in \eqref{cif6}, we finish the proof.

\end{proof}

\begin{lemma}\label{lem:lambda} Let $\varepsilon>0$ and $\xi>1$ be arbitrary. Assume that $T\Delta \geq 2$ and \[\lambda \geq 2\frac{\xi+1}{\xi-1}\log(K/\varepsilon)\sqrt{\frac{\Delta}{{T}}},\] where
$K=2(2+e^2)d(d-1)$.
 Then  we have
\[
 \Pr \left(\left| \nabla\ell(\beta)\right| _{\infty}\leq \frac{\xi-1}{\xi+1}\lambda\right)\geq 1-\varepsilon\; .
\]
\end{lemma}

\begin{proof}
The function \eqref{def: loglik_beta} can be also expressed in the following form
\begin{equation}
\label{l_decomp}
\ell (\theta)  =\frac{1}{T}\sum_{w\in\V} \sum_{ s\not=s'} \lssw (\tssw) ,
\end{equation} 
where 
$$
\lssw (\tssw)  = \sum_{c\in\X_{-w}} \left[-n_w(c;\; s,s\p){\theta_{s,s\p}^{w}}^{\top} Z_w(c)+t_w(c;\; s)\exp({\theta_{s,s\p}^{w}}^{\top} Z_w(c))\right].
$$
We can calculate derivatives
\begin{equation}
\label{der1}
\nabla \lssw (\tssw) =\sum_{c\in\X_{-w}} [-n_w(c;\; s,s\p)+t_w(c;\; s)\exp({\theta_{s,s\p}^{w}}^{\top} Z_w(c))] Z_w(c).
\end{equation}
By Remark \ref{remark_intercept} the matrix $\theta$ has $2d$-rows and $(d-1)$-columns. It can be also considered as a 
$2d(d-1)$-dimensional vector $\left(\theta_{0,1}^{w_1 \top}, \theta_{1,0}^{w_1 \top},
\theta_{0,1}^{w_2\top}, \theta_{1,0}^{w_2 \top},
\ldots, \theta_{0,1}^{w_d \top }, \theta_{1,0}^{w_d \top}\right)^ \top$, where $(w_1, w_2,\ldots, w_d)$ is a fixed order of the nodes of the graph. Using this order we obtain  
\begin{equation}\label{def: gradient}
\nabla \ell (\theta) = \frac{1}{T} \left[ \nabla \lssw (\tssw)
\right]_{w \in \V,s\neq s'}.
\end{equation}
Note that by \eqref{def: gradient}, \eqref{def: beta} and \eqref{der1} we have the following inequality
\begin{equation}\label{eq: bound_ell_infty}
| \nabla\ell(\beta) |_\infty \leq \frac{1}{T} \max_{w \in \V, s\not= s\p, 1\leq k\leq d-1} \left \vert \sum_{c\in\X_{-w}\colon Z_w(c)[k]=1}\left[ n_w(c;\;s,s\p)-t_w(c;\;s)Q_w(c;\;s,s\p)\right]\right\vert\;,
 \end{equation}
where $Z_w(c)[k]$ is  the $k$-th coordinate of $Z_w(c)$ for each $w \in \V, c \in \X _{-w}.$
The core element of the proof is to show that for fixed $w \in \V, s\not= s\p, 1\leq k\leq d-1$ and $\eta>0$
\begin{equation}
\label{core}
  \Pr\left(\left \vert \sum_{c\in\X_{-w}\colon Z_w(c)[k]=1}\left[ n_w(c;\;s,s\p)-t_w(c;\;s)Q_w(c;\;s,s\p)\right]\right\vert>\eta\sqrt{T\Delta} \right)\leq 
  (2+e^2)\exp\left(-\frac{\eta}{2}\right).
 \end{equation}
Having \eqref{core} we finish the proof of Lemma \ref{lem:lambda} using union bounds. Therefore, we focus on proving \eqref{core} that is based on the martingale arguments, so we make the dependence on the time explicit in \eqref{core}, that is $n_w(c;\;s,s\p)$ and $t_w(c;\;s)$ become $n_w^T(c;\;s,s\p)$ and $t_w^T(c;\;s),$ respectively. 

 For $t \in [0,T]$ we define a process
\begin{equation}
\label{mart}
 M(t) = \sum_{c\in\X_{-w}\colon Z_w(c)[k]=1}\left[ n_w^t(c;\;s,s\p)-t^t_w(c;\;s)Q_w(c;\;s,s\p)\right]\;.
\end{equation}
We use the upper index ,,$t$'' in $n_w^t(c;\;s,s\p)$ and $t^t_w(c;\;s)$ to indicate that they correspond to the time interval $[0,t].$
Using Proposition~\ref{prop:martingale}, which is stated below, we obtain that the process $\{M(t): t \in [0,T]\}$ is a martingale. Let us define its jumps by $$\Delta M(t)=M(t)-M(t_-)=\sum_{c\in\X_{-w}\colon Z_w(c)[k]=1}\Ind \left[X(t_-)=(s,c),X(t)=(s\p,c)\right],$$ where $M(t_-)$ is the left limit at $t$. 
By \citep[Theorem II.37]{Protter2005} and \citep[Theorem I.4.61]{Jacod2003} for any $x>-1$ the process
\begin{eqnarray*}
\mathcal{E}_x(t)&=&\exp\left(xM(t)\right)\prod_{u\leq t}(1+x\Delta M(u))\exp(-x\Delta M(u))\\&=&\exp\left\{ xM(t)-(x-\log(1+x))n^t_{s,s\p}\right\}
\end{eqnarray*}
is also a  martingale, where $n^t_{s,s\p}=\sum_{c\in\X_{-w}\colon Z_w(c)[k]=1} n^t_w(c;\;s,s\p)$ is computed for a trajectory at the time interval 
$[0,t]$. Therefore, by Markov inequality together with the triangle inequality we get for any $x\in(0,1]$ 
\begin{eqnarray}
 \label{eq:decomp_prob}
 \Pr(|M(T)|>L) &\leq&  \Pr(|xM(T)-(x-\log(1+x))n^T_{s,s\p}|>xL/2) \nonumber \\ 
&+& \Pr( (x-\log(1+x))n^T_{s,s\p}>xL/2) \nonumber \\ &\leq& 2\exp\left(\frac{-xL}{2}\right)+\Pr( (x-\log(1+x))n^T_{s,s\p}>xL/2)\;.
\end{eqnarray}
We observe that $n^T_{s,s\p}$ is upper bounded by the total number of jumps up to time $T$, which in turn is bounded by a Poisson random variable $N(T)$ with the intensity $T\Delta$. 
Hence,
\[
 \Pr( (x-\log(1+x))n^T_{s,s\p}>xL/2)\leq \exp\left[\frac{-xL}{2}+T\Delta\left(\frac{e^x}{1+x}-1\right) \right]\;.
\]
Applying inequality $e^x\leq 1/(1-x)$ for $x<1$ and setting $x=1/\sqrt{T\Delta}$ we get
\[
 \Pr( (x-\log(1+x))n^T_{s,s\p}>xL/2)\leq \exp\left(\frac{-L}{2\sqrt{T\Delta}}+\frac{T\Delta}{T\Delta-1} \right)\;.
\]
We use $T\Delta\geq 2$ and we plug in $L=\eta\sqrt{T\Delta}$ to conclude the proof.

\end{proof}

\begin{proposition}
 \label{prop:martingale}
 Let $X(t)$ be a Markov jump process with a bounded intensity matrix $Q,$ then
 \[
  M_\nu(t)=n^t_{s,s\p}-t^t_s Q(s,s\p)
  \]
is a martingale with respect to the natural filtration $\mathcal{F}_t$, where $n^t_{s,s\p}$ is a number of jumps from $s$ to $s\p$ on the interval $[0,t]$ and $t^t_s$ is an occupation time at state $s$ on the interval $[0,t]$. The notation
$M_\nu(t)$ means that the distribution at time $0$ is $\nu$.
\end{proposition}
\begin{proof}
 Since $\Ex(M_\nu(t)|\mathcal{F}_u) = M_\nu(u)+\Ex (M_{X(u)}(t-u)|X(u))$ for any $u<t,$ it is enough to show that for all $t>0$ and all initial measures $\nu$ we have $\Ex M_\nu(t)=0$.

Let $k_0^n,\dots k_n^n$ be defined for any $n\in\mathbb{N}$ by $k_i^n = ti/n$ for all $i=0,\dots, n$. 
Since the trajectory of the process is \textit{c\`adl\`ag}, we have
 \begin{equation*}
  \Ex M_\nu(t) = \Ex \lim_{n\to\infty}\sum_{i=1}^n\left[\Ind(X(k^n_{i-1})=s,X(k^n_i)=s\p) - \frac{t}{n}Q(s,s\p)\Ind(X(k_{i-1}^n )= s)\right]\,.
 \end{equation*}
 We observe that for all $n\in\mathbb{N}$
 \begin{equation}
\label{mart1}
\left|\sum_{i=1}^n\left[\Ind(X(k^n_{i-1})=s,\Ind(X(k^n_i)=s\p) - \frac{t}{n}Q(s,s\p)\Ind(X(k_{i-1}^n )= s)\right]\right|\leq N(t)+t\;
   \end{equation}
 where $N(t)$ is the total number of jumps. Since $N(t)$ is a Poisson process with a  bounded intensity, 
 the right-hand side of \eqref{mart1} is integrable and by the dominated convergence theorem and the definition of $Q$ we get
\begin{align*}
  \Ex M_\nu(t) &= \lim_{n\to\infty}\Ex \sum_{i=1}^n\left[\Ind(X(k^n_{i-1})=s,X(k^n_i)=s\p) - \frac{t}{n}Q(s,s\p)\Ind(X(k_{i-1}^n )= s)\right]\\
  &=\lim_{n\to\infty}\Ex \sum_{i=1}^n\left[\Ex(\Ind(X(k^n_{i-1})=s,X(k^n_i)=s\p)|X(k_{i-1})) - \frac{t}{n}Q(s,s\p)\Ind(X(k_{i-1}^n )= s)\right]\\
  &= \lim_{n\to\infty}\Ex \sum_{i=1}^n\left[\left(\frac{t}{n}Q(s,s\p)-o(1/n)\right)\Ind(X(k_{i-1}^n )= s) - \frac{t}{n}Q(s,s\p)\Ind(X(k_{i-1}^n )= s)\right]\\
  &=\lim_{n\to\infty}\Ex \sum_{i=1}^no(1/n)\Ind(X(k_{i-1}^n )= s)=0
 \end{align*}
 \end{proof}

 \begin{lemma}
 \label{lem:cif}
Let $\varepsilon \in (0,1), \xi >1$ be arbitrary.  Suppose that $F(\xi)$ defined in   \eqref{ass:cif} is positive and 
  \begin{equation}
\label{Tform} 
 T>
\frac{36 \left[ (\max\limits_{w \in \V} |\I _w| +1) \log 2 + \log\left(
d  ||\nu||_2 / \varepsilon
\right)  
\right]}{\min\limits_{w \in \V,s \in\X_w ,c_{\I _w}\in\X_{\I _w}} \pi^2(s,c_{\I _w},0)\rho_1} ,
\end{equation}
 then
 \[
 \Pr\left(\bar F(\xi) \geq \zeta F(\xi) \right)\geq 1-\varepsilon,
 \]
where  $ \zeta=\min\limits_{w \in \V, s \in \X _w,c_{\I _w}\in\X_{\I _w}} \pi(s,c_{\I _w},0)/2.
    $ 
\end{lemma}

\begin{proof}
 By the definition of $\bar F(\xi),$ \eqref{ass:cif} and by the formula for Hessian of $\ell$ (see \eqref{eq:draddiff2}) we have
 \begin{equation}
\label{Fprop}
  \frac{\bar F(\xi)}{F(\xi)}\geq \frac{1}{T}\min_{w \in \V, s,c_{\I _w}\in\X_{\I _w}}t_w((c_{S_w},0);s) \;.
 \end{equation}
We complete the proof by lower bounding the right-hand side of 
\eqref{Fprop}.
First, we can calculate that
 \begin{eqnarray}
\label{form1}
 &\,& \Pr\left(\min_{w \in \V, s\in \X _w,c_{\I _w}\in\X_{\I _w}} \quad \frac{1}{T}t_w((c_{S_w},0);s)\geq \zeta\right) \nonumber \\
&\geq& \Pr\left(\forall_{w \in \V, s\in \X _w,c_{\I _w}\in\X_{\I _w}} \quad \quad \frac{1}{T}t_w((c_{S_w},0);s)\geq \pi(s,c_{\I _w},0)/2
\right) \nonumber
\\ &\geq& 1- 2d  \max_{w \in \V, s \in \X _w,c_{\I _w}\in\X_{\I _w}}
    2 ^{|\I _w|} \Pr\left(\frac{1}{T}t_w((c_{S_w},0);s)< \pi(s,c_{\I _w},0)/2 \right).
 \end{eqnarray}
Using  Lemma~\ref{lem:Lezaud} (given below)
we lower bound \eqref{form1} by 
$$
1- 2d  \max_{w \in \V, s \in \X _w,c_{\I _w}\in\X_{\I _w}}
    2 ^{|\I _w|} \Vert\nu \Vert_2\exp\left(-\frac{\pi^2(s,c_{\I _w,0})\rho_1 T}{16+20\pi(s,c_{\I _w},0)}\right).
$$
Applying \eqref{Tform}, we conlude the proof.
\end{proof}

The next lemma is a direct application of \citep[Theorem 3.4]{Lezaud1998}.

\begin{lemma}\label{lem:Lezaud}
For any $w \in \V,s \in\X_w, c_{\I _w}\in\X_{\I _w}$
\[\Pr\left(\frac{1}{T}t_w((c_{S_w},0);s)\leq\pi(s,c_{\I _w},0)/2\right)\leq \Vert\nu \Vert_2\exp\left(-\frac{\pi^2(s,c_{\I _w},0)\rho_1 T}{16+20\pi(s,c_{\I _w},0)}\right).
 \]
\end{lemma}
\begin{proof}
Fix $w \in \V,s \in\X_w, c_{\I _w}\in\X_{\I _w}.$
By the definition we have
 \begin{align*}
  t_w((c_{S_w},0);s)&=\int_0^T  \Ind\left( X(t) =((c_{\I_w},0),s) \right)dt;.
 \end{align*}
Let us define $f(X(s)) = \pi(c_{\I_w},s,0)-\Ind\left( X(t) =(c_{\I_w},0,s) \right)$.
Taking $\gamma = \pi(c_{\I_w},s,0)/2$ in \citep[Theorem 3.4]{Lezaud1998}, we conclude the proof. 
\end{proof}

\begin{lemma}
\label{basiclem}
Let $\tilde{\beta} = \hat{\beta} - \beta$, $z^* = |\nabla \ell(\beta)|_\infty.$ Then 
\begin{equation}
\label{basic}
(\lambda - z^*) |\tilde{\beta}_{S^c}|_1 \leq \tilde{\beta} ^\top \left[ \nabla \ell (\hat{\beta})
- \nabla \ell (\beta)\right] + (\lambda - z^*) |\tilde{\beta}_{S^c}|_1 
\leq (\lambda + z^*) |\tilde{\beta}_{S}|_1\, .
\end{equation}
Besides, for arbitrary $\xi >1$ on the event 
\begin{equation}
\label{omega1}
\Omega_1=\left\{ |\nabla \ell (\beta)|_\infty \leq \frac{\xi -1}{\xi +1} \lambda \right\}
\end{equation} 
 the random vector $\tilde{\beta}$ belongs to the cone $\cone (\xi,S).$ 
\end{lemma}
The proof of Lemma \ref{basiclem} is omitted, because it is similar to the proof of \citet[Lemma 3.1]{Cox13} and is based on convexity of $\ell (\theta)$ and properties of the Lasso penalty.

\begin{lemma}
\label{estim}
Let $\xi >1 $ be arbitrary.  Moreover,
let us denote $\tau = \frac{(\xi +1) |S| \lambda }{ 2\bar{F}(\xi) }$ and an event
\begin{equation}
\label{omega2}
\Omega_2=\left\{\tau < e^{-1} \right\}\,.
\end{equation}
Then $\Omega_1 \cap \Omega_2 \subset A,$ where
\begin{equation}
\label{estim1}
A= \left\{|\hat{\beta} - \beta| _\infty \leq \frac{2 \xi e^\eta \lambda}{(\xi+1) 
  \bar{F}(\xi)} \right\} 
\end{equation}
and $\eta < 1 $ is the smaller solution of the equation $\eta e ^{- \eta} = \tau.$
\end{lemma}
The proof of Lemma \ref{estim} is omitted, because it is similar to 
\citet[Theorem 3.1]{Cox13} or \citet[Lemma 6]{Rejchel18}. In this proof  we use the following 
analog of \citet[Lemma 5]{Cox13}.

\begin{lemma}\label{lem:second_deriv}
 For any $b \in \mathbb{R}^{2d(d-1)}$ we define  $c_b = \max\limits_{w \in \V, s \neq s', c \in \X _{-w} }\exp\left(\left|
b^{w\top}_{s,s'}  Z_w(c)
\right|\right).$ Then
we have
 \begin{equation}\label{inlem}
  c_b^{-1} b^{\top}\nabla^2\ell(\beta) b \leq b^{\top}[\nabla\ell(\beta+b)-\nabla\ell(\beta)]\leq c_b b^{\top}\nabla^2\ell(\beta) b
 \end{equation}
and
\begin{equation}\label{inlem2}
 c_b^{-1} \nabla^2\ell(\beta) \leq \nabla^2\ell(\beta+b) \leq c_b \nabla^2\ell(\beta) \;,
\end{equation}
where for two symmetric matrices $A,B$ the expression $A \leq B$ means that $B-A$ is a nonnegative definite matrix.
\end{lemma}
\begin{proof}
We prove only the inequality \eqref{inlem}, because \eqref{inlem2}  can be established similarly.

The gradient of $\ell(\theta)$ can be computed as in \eqref{def: gradient} in the proof of Lemma \ref{lem:lambda}. 
By the same way we calculate second derivatives
$$
\nabla ^2 \lssw (\tssw) =\sum_{c\in\X_{-w}} t_w(c;\; s)\exp(\theta_{s,s\p}^{w\top} Z_w(c)) Z_w(c) Z_w(c)^ \top.
$$
The second derivative of $\ell (\theta)$ consists of matrices $\frac{1}{T} \nabla ^2 \lssw (\tssw)
$ along its diagonal and zeroes elsewhere. 
Therefore, we have
	\begin{multline}\label{eq:draddiff}
		b^\top[\nabla\ell(\beta + b)-\nabla\ell(\beta)] =\\ \frac{1}{T} \sum_{w\in\V}\sum_{c\in\X_{-w}} \sum_{ s\p\not=s} t_w(c;\; s) \left(b_{s,s\p}^{w\top} {Z}_w(c)\right) \exp(\beta_{s,s\p}^{w\top} Z_w(c))
\left[\exp(b_{s,s\p}^{w^\top} Z_w(c))-1\right]
	\end{multline}
as well as
\begin{equation}
\label{eq:draddiff2}
b ^T \nabla ^2 \ell(\beta) b =
\frac{1}{T} \sum_{w\in\V}\sum_{c\in\X_{-w}} \sum_{ s\p\not=s} t_w(c;\; s) \left(b_{s,s\p}^{w\top} {Z}_w(c)\right)^2 \exp(\beta_{s,s\p}^{w\top} Z_w(c)).
\end{equation}
	Let us consider an arbitrary summand in the sum \eqref{eq:draddiff} and the corresponding one in \eqref{eq:draddiff2}. We can focus only on cases where $t_w(c;\; s) >0$ and  $b_{s,s\p}^{w\top} {Z}_w(c) \neq 0.$ 
From the mean value theorem we obtain for all nonzero $x\in\mathbb{R}$
	\begin{equation}
\label{expon}
	\e^{-|x|}\leq \frac{\e^x-1}{x}\leq \e^{|x|}\:.
	\end{equation}
	Using \eqref{expon} we can write
	\begin{equation}\label{dineq}
	t_w(c;\; s)\exp({-|b_{s,s\p}^{w\top}Z_w(c)|})\leq \frac{t_w(c;\; s)(\exp({b_{s,s\p}^{w\top}Z_w(c)})-1)}{b_{s,s\p}^{w\top}Z_w(c)}\leq t_w(c;\; s) \exp({|b_{s,s\p}^{w\top}Z_w(c)|}).
	\end{equation}
Finally, we multiply each side of \eqref{dineq} by $(b_{s,s\p}^{w\top}Z_w(c))^2
\exp({\beta_{s,s\p}^{w\top}Z_w(c)})$ to conclude the proof.

\end{proof}

\section{Proofs of main results}

\begin{proof} [of Theorem~\ref{thm:consistency}]
Fix $\varepsilon >0$ and $\xi>1.$ From Lemma \ref{lem:cif} we know that $
 \Pr\left(\bar F(\xi) \geq \zeta F(\xi) \right)\geq 1-\varepsilon.$ Using it with the right-hand side of \eqref{lambda_form} we obtain that $P(\Omega_2) \geq 1- \varepsilon. $ Moreover, from Lemma \ref{lem:lambda} we have that $P(\Omega_1) \geq 1- \varepsilon. $ Therefore, Lemma \ref{basiclem} and \ref{estim} (with $\eta=1$ for simplicity) give us that with probability at least $1-2 \varepsilon$
$$\left\{|\hat{\beta} - \beta| _\infty \leq \frac{2 \xi e \lambda}{(\xi+1) 
  \bar{F}(\xi)} \right\}\,.$$ 
Finally, we again bound $\bar F(\xi)$ from below by $\zeta  F(\xi).$

\end{proof}

\begin{proof} [of Corollary \ref{thm:consistency2}]
The proof is a simple consequence of the uniform bound \eqref{estim_formula}  obtained in Theorem \ref{thm:consistency}. Indeed, for an  arbitrary 
$w \in \V, s \neq s'$ and the coordinate $j$ such that $\beta_{s,s\p}^{w} (j) =0$
 we obtain 
\[
| {\hat {\beta}_{s,s\p}}^{w} (j)| =|\hat \beta_{s,s\p}^{w} (j) -  \beta_{s,s\p}^{w} (j)| \leq |\hat \beta - \beta|_\infty \leq \delta.
\] 
Analogously, for each 
$w \in \V, s \neq s'$ and the coordinate $j$ such that $\beta_{s,s\p}^{w} (j)  \neq 0$
 we have 
\[
|\hat  \beta_{s,s\p}^{w} (j)  | \geq  | \beta_{s,s\p}^{w} (j)| -|\hat  \beta_{s,s\p}^{w} (j)  -  \beta_{s,s\p}^{w} (j) | \geq \beta_{\min} - |\hat \beta - \beta|_\infty > 2 \delta  - R \geq \delta . 
\]
\end{proof}

\end{document}